\documentclass[letterpaper, 10 pt, conference]{ieeeconf}  
                                                          
\IEEEoverridecommandlockouts                              

\usepackage{xcolor}

\usepackage{graphics} 
\usepackage{epsfig,epstopdf} 
\usepackage{mathptmx} 
\usepackage{times} 
\usepackage{amsmath,scalerel}

\DeclareMathOperator*{\argmin}{arg\,min}
\usepackage{amssymb}  
\usepackage{amsbsy}
\usepackage{comment}
\usepackage{relsize}
\usepackage{graphicx}      
\usepackage{balance}
\usepackage{lipsum}
\usepackage{ifpdf}
\usepackage[mathcal]{euscript}
\usepackage[T1]{fontenc}
\usepackage{enumerate}
\usepackage{amsfonts}
\usepackage{mathtools}
\usepackage{soul, color}

\newcommand{\redoB}[1]{\textcolor{blue}{#1}}

\newtheorem{thm}{Theorem}
\newtheorem{lemma}{Lemma}
\newtheorem{remark}{Remark}
\newtheorem{assum}{Assumption}

\usepackage[bottom]{footmisc}
\usepackage{cite}
\usepackage{graphicx}
\usepackage{longtable}
\graphicspath{{figures/}} 
\usepackage{multirow}
\usepackage{algorithm}
\usepackage{algorithmic}
\floatname{algorithm}{Algorithm 1:}

\usepackage{tikz}
\usetikzlibrary{arrows}
\usetikzlibrary {positioning}
\usetikzlibrary{fit,calc,shapes.geometric}
\tikzset{edge/.style = {->,> = latex}}
\usepackage{subfigure}
\usepackage[10pt]{moresize}
\usepackage{caption}

\title{\LARGE \bf
Federated Learning in Wireless Networks via Over-the-Air Computations*

\author{Halil Yigit Oksuz$^{1,2,3}$, Fabio Molinari$^{1}$, Henning Sprekeler$^{2,3}$, J\"{o}rg Raisch$^{1,2} $ }

\thanks{* This work has been funded by the Deutsche Forschungsgemeinschaft (DFG, German Research Foundation) under Germany’s Excellence Strategy – EXC 2002/1 “Science of Intelligence” – project number 390523135.}

\thanks{$^{1}$ H.Y. Oksuz, F. Molinari, and J. Raisch are with the Control Systems Group at Technische  Universität Berlin, Germany. {\tt\small \{oksuz@tu-berlin.de\}, \{molinari,raisch\}@control.tu-berlin.de}. } 

\thanks{$^{2}$ H.Y. Oksuz, H. Sprekeler, and J. Raisch are with Exzellenzcluster Science of Intelligence, Technische Universität Berlin, Marchstr. 23, 10587, Berlin, Germany.}

\thanks{$^{3}$ H.Y. Oksuz and H. Sprekeler are with the Modelling Cognitive Processes Group at Technische  Universität Berlin, Germany.
}
}

\begin{document}

\maketitle
\thispagestyle{empty}
\pagestyle{empty}

\begin{abstract}                
	In a multi-agent system, agents can cooperatively learn a model from data by exchanging their estimated model parameters, without the need to exchange the  locally available data used by the agents. This strategy, often called \textit{federated learning}, is mainly employed for two reasons: (i) improving resource-efficiency by avoiding to share potentially large data sets and (ii) guaranteeing privacy of local agents' data.  Efficiency can be further increased by adopting a beyond-5G communication strategy that goes under the name of Over-the-Air Computation. This strategy exploits the interference property of the wireless channel. Standard communication schemes prevent interference by enabling transmissions of signals from different agents at distinct time or frequency slots, which is not required with Over-the-Air Computation, thus saving resources. In this case, the received signal is a weighted sum of transmitted signals, with unknown weights (fading channel coefficients). State of the art papers in the field aim at reconstructing those unknown coefficients. In contrast, the approach presented here does not require reconstructing channel coefficients  by  complex encoding-decoding schemes. This improves both efficiency and privacy.
\end{abstract}


\vspace{2mm}
\section{INTRODUCTION}
\vspace{1mm}

Over the last decade, topics related to machine learning have attracted a great deal of attention due to their success in many application areas (e.g., \cite{lecun2015deep,goodfellow2016deep,nedic2020distributed}). As computational power increases, we have smaller but more intelligent and powerful devices, which are able to handle big volumes of data and more complex computations.

Data is often distributed over these intelligent devices or agents, such as smartphones, personal computers and distributed data centers. In a centralized learning approach, all data are collected in a computationally powerful machine, such as a cloud-based unit or a server, on which the model is trained \cite{chen2019deep}. However, this approach may not be suitable for specific applications requiring privacy, resource efficiency, and  low-latency. For instance, due to privacy concerns, smartphone users might not be willing to share their data; moreover, with an increasing number of agents, more communication resources must be allocated to share the agents' data with the central unit. However, it is also possible to train a model without centralized data collection. In this case, agents individually train models on their respective data and share their individually trained models (in the form of parameter vectors) with the central unit. Sharing parameter vectors is, in general, less expensive in terms of communication resources than sharing all the local data sets. This approach is often referred to as \textit{federated learning}, and allows more private and communication efficient learning (see \cite{mcmahan2017communication, smith2017federated,yang2019federated,bonawitz2019towards}).

Even though federated learning allows improved privacy by avoiding to share individual data sets, sharing local model parameters with the central unit might still reveal sensitive information \cite{mcmahan2017learning, kairouz2021advances, li2020federated}. It is possible to tackle this privacy problem with popular techniques like encryption or differential privacy \cite{dwork2006calibrating}, but they come at the price of lower efficiency and performance  \cite{mcmahan2017learning,bonawitz2017practical}. 

As the number of agents and the number of model parameters increase, the communication load on the overall system increases as well \cite{konevcny2016federated,li2020federated,kairouz2021advances}.  To cope with this problem, one can carry out local training on agents for multiple steps or utilize \textit{quantization} and/or \textit{sparsification} on the model updates in order to accomplish an efficient compression; another approach is allowing only a subset of agents to transmit at specific time steps \cite{alistarh2017qsgd,alistarh2018convergence,reisizadeh2020fedpaq,chen2021communication,mitra2021linear}. However, most of these techniques are not resource efficient in the sense that they increase the need of bandwidth or the number of communication rounds, which in general leads to a decrease in total throughput and learning speed as also observed in \cite{molinari2021max,frey2021over}. 

Instead, we consider the so-called Over-The-Air computation approach, e.g., \cite{zhu2021over}, to improve communication efficiency. When multiple agents transmit at the same time and in the same frequency band, signals are affected by the physical phenomenon of interference. Standard communication protocols prevent interference by transmitting signals orthogonally: in TDMA (Time Division Multiple Access), agents are assigned different time slots when they can transmit, whereas in FDMA (Frequency Division Multiple Access), different frequency bands are allocated to different users.

The philosophy of Over-The-Air Computation is to exploit interference rather than combat it, to increase communication efficiency. For example, in
a system composed of $N$ agents, each of which has to transmit an $m$-dimensional parameter vector $\theta \in \Re^{m}$ to a central unit, 
TDMA or FDMA would require $mN$ orthogonal transmissions (multiplexed in time or frequency), whereas Over-The-Air Computation requires only $m$ orthogonal transmissions. Due to its advantages, federated learning has been carried out by \cite{yang2020federated,amiri2020machine,sery2021over} via over-the air computation.

With this in mind, the main contributions of this paper can be summarized as follows:

\begin{itemize}
\item  Unlike the studies by \cite{yang2020federated} and \cite{sery2021over}, we do not assume to know (nor to be able to reconstruct) channel coefficients, but we present an algorithm
that can deal with their unknown nature. We will not need extra pre-processing to reconstruct the channel,  which makes the proposed scheme more time and resource efficient. 
\item  Privacy is inherently guaranteed by the unknown nature of the channel. 
The central unit will not be able to reconstruct information transmitted by individual agents.
\end{itemize}

The remainder of this paper is organized as follows: we describe the problem setup in Section II. In Section III, we present the proposed federated learning algorithm and prove its convergence. A numerical example is presented in Section IV. Concluding remarks are given in Section V.

\subsection*{Notation}

The set of real numbers is denoted by $\Re$, $\Re^{m}$ represents $m$-dimensional Euclidean space. $\mathbb{N}$ and $\mathbb{N}_0$ respectively denote the set of natural numbers and the set of nonnegative integers. For a vector $x\in \Re^{m}$, $x^{T}$ denotes its transpose.  The 2-norm of vector $x$ is denoted by $||x||$. The expected value of a random variable $p$ is denoted by $\mathbb{E}[p]$. Given a finite set $S$, its cardinality is denoted by $|S|$.  For a differentiable function $f: \Re^{m}\to \Re$, $\nabla f(\mathbf{x})$ represents the gradient of the function $f$ at $x\in \Re^{m}$. The projection of  $x\in \Re^{m}$ onto a nonempty closed convex set $S\subset \Re^{m}$ is denoted by $\mathbf{P}_{S}(x)$, where $\mathbf{P}_{S}(x) = \argmin_{s\in S}|| s-x ||$.

\vspace{2mm}
\section{Problem Description}

\subsection{Federated Learning with Constraints}

Consider a system of $N$ agents connected to a server or cloud based unit, whose objective is to carry out a machine learning task. Each agent can access different portions of the dataset, and has an individual local cost function. We denote the set of data available to the $i$-th agent by $D_i = \{d^{n}_i \}_{n=1}^{ |D_i|}$ and use $\mathcal{L}_i(d^{n}_i,\mathbf{\theta})$ to represent the value of the cost function of a model with parameter $\mathbf{\theta}\in \mathcal{R}^m $. In a supervised learning setting, the dataset $D_i$ consists of pairs of inputs and targets, i.e., $d^{n}_i = (u^{n}_i , z^{n}_i )$,  where $u^{n}_i$ and $z^{n}_i$ represent, respectively, input and  target data.  Moreover, the private  local cost function of agent $i$ can be expressed as \\[-2mm]
\begin{equation}\label{eq:local_cost}
f_i(\theta)= \frac{1}{|D_i|}\sum_{n=1}^{|D_i|} \mathcal{L}_i(d^{n}_i,\theta),
\end{equation} 

If the global cost function is defined as the average of all local cost functions, i.e., \\[-2mm]
\begin{equation}\label{eq:def_global}
\mathcal{L}(\theta) = \frac{1}{N}\sum_{i=1}^{N} f_i(\theta),
\end{equation} 
\noindent
then, the objective of the overall system is to cooperatively solve the following constrained optimization problem \\[-2mm]
\begin{equation}\label{eq:global_from_local}  
\theta^{*} = \argmin_{\theta\in \Theta }  \mathcal{L}(\theta) =\argmin_{\theta\in \Theta }\frac{1}{N} \sum_{i=1}^{N} f_i(\theta),
\end{equation} 
\noindent
where $\Theta  \subset \Re^m$ is a nonempty constraint set. In what follows, we denote $\mathcal{L}(\theta^{*})$ by $\mathcal{L}^{*}$. Moreover, the set of optimal solutions is defined as \\[-2mm] $$\Theta^*=\{ \theta \in \Theta \textbf{\hspace{1mm}}|\textbf{\hspace{1mm}} \frac{1}{N}\sum_{i=1}^{N} f_i(\theta)   =\mathcal{L}^{*}\}.$$
If the entire dataset, i.e., \\[-2mm] 
$$D = D_1 \cup  D_2  \cup \vspace{1mm} \cdots  \cup D_N$$ 
were known to the server, one could employ a centralized gradient descent based optimization to solve the above global learning task \cite{nedic2020distributed}.
However, in the considered federated learning approach, each agent can access only its own dataset, on which it trains its own model. After local optimization steps, the current versions of the local parameter estimates are transmitted to the server, where they are aggregated. The aggregated version is then transmitted to the agents for the next optimization step (see Fig. \ref{fig:my_label}).

\begin{figure}
\includegraphics[scale=0.31]{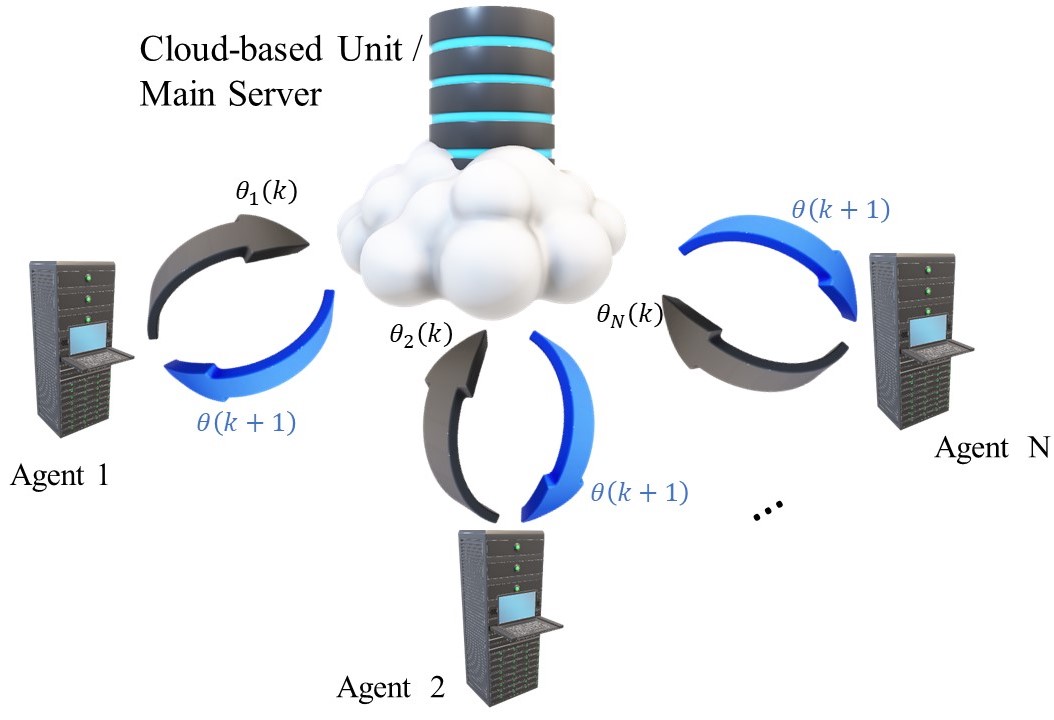}
\caption{An illustration of a federated learning setting.}
\label{fig:my_label}
\end{figure}

\subsection{Over-the-Air Communication}\label{sec:wmac}

In wireless communication systems, the wireless multiple access channel (WMAC) model has been extensively used to characterize communication between multiple transmitters 
and a single receiver over fading channels, e.g.,  \cite{ahlswede1973multi,Giridhar2006}.
Throughout this paper, we employ the WMAC-based communication model described
by~\cite{molinari2021max,molinari2022over}. In this model, multiple (say $N$) agents simultaneously send information $\{\mathbf{s}_i(k)\in\Re^{m}\}_{i=1}^N$ in the same frequency band. This information is corrupted  by the channel and superimposed by the receiver, i.e., the received information is \\[-2mm]
\begin{equation}\label{eq:interference}
\mathbf{r}^{rec}(k) =\sum_{i=1}^{N}\alpha_{i}(k)\mathbf{s}_i(k),
\end{equation}
\noindent
where the $\alpha_{i}(k)$ are unknown time-varying positive channel coefficients.

As indicated above, employing Over-the-Air computation for federated learning provides the following benefits:

\begin{itemize}
\item \textbf{privacy}: the channel coefficients $\{\alpha_i(k)\}_{i=1}^N$ in (\ref{eq:interference})
are unknown, thus it is impossible for
the central unit to individually reconstruct $\{s_i(k)\}_{i=1}^N$ from $\mathbf{r}^{rec}(k)$;
\item \textbf{efficiency}: as shown by \cite{frey2021over}, Over-the-Air Computation can achieve better-than-linear scaling of the communication cost as the number of transmitters grows.
\end{itemize}

Unlike the algorithm suggested by \cite{sery2021over}, our approach does not require any knowledge of the channel coefficients, nor do they need to be reconstructed at each time step, thus removing one source of complexity.

\begin{algorithm}[tb]
\caption{FedCOTA}
\begin{algorithmic}[1]
	\REQUIRE  $\theta(0)\in\Theta$.\vspace{0.5mm}
	\ENSURE
	\FOR {each time step $k\in \mathbb{N}_0$}
	\STATE \text{The server broadcasts $\theta(k)$} 
	\STATE \text{Each agent $i$ updates its local variables:}
	\begin{align}
		\label{eq:local_agent_update}
		&\theta_i(k) = \theta(k) -  \eta(k)\nabla f_i\big(\theta(k)\big)\\ 
		&\varrho_i(k)=1
	\end{align}\vspace{-2mm}
	\STATE \text{Each agent $i$ transmits $\theta_i(k)$ and $\varrho_i(k)$}
	\STATE \text{The server receives:}
	\begin{align}
		\label{eq:ser_rec_1}
		\theta^{\text{rec}}(k)&=\sum_{i=1}^{N}\alpha_{i}(k)\theta_i(k)\\  
		\label{eq:ser_rec_2}
		\varrho^{\text{rec}}(k)&=\sum_{i=1}^{N}\alpha_{i}(k)\varrho_i(k)
	\end{align}
	\STATE \text{The server updates:}
	\begin{align}\label{eq:server_comp}
		\theta(k+1) &= \mathbf{P}_\Theta\Bigg(\frac{\theta^{\text{rec}}(k)}{\varrho^{\text{rec}}(k)}\Bigg)
	\end{align}
	\ENDFOR
\end{algorithmic}
\label{alg:oafl}
\end{algorithm}

\section{Federated Constrained Over-the-Air Learning (FedCOTA) Algorithm}
\label{sec:TheRes}

In this section, we first introduce the proposed algorithm that allows agents to carry out federated learning via over-the-air communication. Then, we will investigate the convergence properties of the proposed algorithm.

\subsection{Algorithm}

The FedCOTA algorithm is summarized in Algorithm 1. At the beginning, the server broadcasts to agents $\theta(0)\in \Theta$. Then, through iterations, each agent computes its own parameter vector $\theta_i(k)$ by using the local update rule (\ref{eq:local_agent_update}).  Afterwards, all agents transmit simultaneously (and in the same frequency band) their local parameter vectors $\theta_i(k)$ to the central unit. Then, again simultaneously and in the same frequency band a constant $\varrho_i(k)=1$ is transmitted by all agents. Because of the superposition property of the wireless channel (see Section~\ref{sec:wmac}), the server receives (\ref{eq:ser_rec_1}) and (\ref{eq:ser_rec_2}) at time step $k$. Finally, the server computes (\ref{eq:server_comp}) thus obtaining $\theta(k+1)$,  which can be written as 
\begin{align}\label{eq:developFEDOTA}
\nonumber
\theta(k+1) &= \mathbf{P}_\Theta\Bigg( \frac{\theta^{\text{rec}}(k)}{\varrho^{\text{rec}}(k)}\Bigg)= \mathbf{P}_\Theta\Bigg(\sum_{i=1}^{N}\frac{\alpha_{i}(k)}{ \sum_{i=1}^{N}\alpha_{i}(k)}	\theta_i(k)\Bigg) \\
&= \mathbf{P}_\Theta\Bigg(\sum_{i=1}^{N}h_{i}(k)	\theta_i(k)\Bigg)
\end{align}
\noindent
where $\mathbf{P}_\Theta$ is the projection operator onto the set $\Theta$,  the $\alpha_{i}(k)$ are the unknown time-varying positive real channel coefficients, and the $h_{i}(k)=\frac{\alpha_{i}(k)}{ \sum_{i=1}^{N}\alpha_{i}(k)}$ are the corresponding normalized channel coefficients. By construction, the $h_{i}(k)$ are positive and 
\begin{equation}\label{eq:eq_asmp_h}
\sum_{i=1}^{N} h_{i}(k)=1,
\end{equation}
\noindent
for all $k\geq 0$.

\subsection{Preliminaries}
\label{sec:TheResStrongcONV}

The following assumptions, also made in similar papers in the field, will hold throughout the paper.

\begin{assum}
\label{ass:conset}
The constraint set $\Theta\subset\Re^{m}$ is convex and compact. As a consequence (see {\cite[Theorem~2.41, p.40]{rudin1976principles}}),   $\Theta$ is then also closed and bounded.
\end{assum}

\begin{assum}
\label{ass:f_i}
The cost functions $f_i(\theta)$  are continuously differentiable and strongly convex with $L$-Lipschitz continuous gradients, i.e., for any $ \theta_1,\theta_2 \in \Theta \subset \Re^{m}$ and for all $i=1,2,\ldots, N$, the following inequalities hold:
\begin{itemize}
	\item[$(i)$] $\exists \mu_i >0 $ such that  
	\begin{align}
		\label{eq:str_conv}
		f_i(\theta_2)-f_i(\theta_1)  \geq  \nabla f_i(\theta_1)^T(\theta_2-\theta_1) + \frac{\mu_i}{2} ||\theta_2-\theta_1||^2, \\[-8mm] \nonumber
	\end{align}
	\item[$(ii)$]$\exists L_i >0 $ such that %
	\begin{align}
		||\nabla f_i(\theta_1)-\nabla f_i(\theta_2)|| \leq L_i ||\theta_1-\theta_2||. \\[-4mm] \nonumber
	\end{align}
\end{itemize}
\end{assum}

\begin{remark}
\label{ass:L_i}
Note that the global cost function $\mathcal{L}(\theta)$  is then also differentiable and strongly convex with $L$-Lipschitz continuous gradient, i.e.,
\begin{align*}
\centering
\mathcal{L}(\theta_2)-&\mathcal{L}(\theta_1)  \geq  \nabla \mathcal{L}(\theta_1)^T(\theta_2-\theta_1) + \frac{\mu}{2} ||\theta_2-\theta_1||^2 \\
&||\nabla \mathcal{L}(\theta_1)-\nabla \mathcal{L}(\theta_2)|| \leq L ||\theta_1-\theta_2||, 
\end{align*}  
\noindent
where  $ \mu = \frac{1}{N}\sum_{i=1}^{N}\mu_i$ and $L = \frac{1}{N}\sum_{i=1}^{N}L_i$ (see \cite{yuan2016convergence}). 
\end{remark}
\vspace{0.5mm}
\begin{remark}
\label{ass:L_i_2}
Note that Assumptions \ref{ass:conset} and \ref{ass:f_i} imply that the local and global cost functions have bounded gradients on $\Theta$.  
\end{remark}

Note that Assumptions \ref{ass:conset} and \ref{ass:f_i}  have been widely  utilized in the federated learning and distributed optimization literature to illustrate the existence of a solution to problem (\ref{eq:global_from_local}) and convergence properties of the proposed algorithms (see  \cite{xi2018linear,xin2020general,mai2016distributed,li2018distributed,nesterov2003introductory,sery2021over,nedic2020distributed}, and the references therein). 

\begin{assum}
\label{ass:step_size_TV}
For all $k\geq 0$, the step size in Algorithm 1 is chosen as $\eta(k)=\frac{\eta_c}{\sqrt{k+1}}$, where $0< \eta_c \leq \frac{1}{L}$. 
\end{assum}

\begin{remark}
\label{rem:step_size_TV}
Under Assumption \ref{ass:step_size_TV}, the step size $\eta(k)$ is  decreasing and satisfies $\sum_{k=0}^{\infty}\eta (k)=\infty$ and $\lim_{k\to\infty}\eta(k) = 0$.
\end{remark}

Conditions similar to Assumption~\ref{ass:step_size_TV} on the step size 
have also been utilized in many previous studies (e.g., \cite{nesterov2003introductory,nedic2014distributed,nedic2020distributed,yang2019survey,wang2017distributed,li2018distributed,xie2018distributed,mai2019distributed}, and the references therein) to show the exact convergence to an optimal solution.

\begin{lemma}
\label{lem:lem_11}
If $f(\cdot)$  is continuously differentiable, convex with $L$-Lipschitz continuous gradients, then for any  $ \theta_1, \theta_2 \in \Re^{m}$, the following inequality holds:
\begin{align*}
0\leq f(\theta_2)-f(\theta_1)- \nabla f(\theta_1)^{T}(\theta_2-\theta_1) \leq \frac{L}{2}||\theta_1-\theta_2||^2.  
\end{align*}
\begin{proof}
The result is a direct consequence of {\cite[Theorem~2.1.5]{nesterov2003introductory}}.
\end{proof}
\end{lemma}
\begin{remark}
	\label{rem:L}
 	From Lemma \ref{lem:lem_11} and the definition of strong convexity (\ref{eq:str_conv}), we have $\mu \leq L$ if a function is continuously  differentiable, strongly convex with $L$-Lipschitz continuous gradients. Hence,  $\eta_c \leq \frac{1}{L} \leq \frac{1}{\mu}$ holds by Assumption \ref{ass:step_size_TV}. 	
\end{remark}

\begin{assum}
	\label{ass:alphas}
	The unknown time-varying positive real channel coefficients are assumed to be independent realizations (across time and agents) of the same probability distribution, i.e., $\forall k \in \mathbb{N}_0$, $\alpha_{i}(k) \sim \mathcal{D}\big(\bar{\alpha},Var(\alpha)\big)$, where $\bar{\alpha}$ and $Var(\alpha)$ respectively denote the mean and variance of the distribution $\mathcal{D}$. 
\end{assum}

As in \cite{molinari2021max,molinari2022over}, we refer here to {\cite[Ch~2.3, Ch~2.4]{Tse2012}} and {\cite[Ch~5.4]{molisch2012wireless}}, thus considering channel coefficients independent realizations of the same probability distribution (see \cite{sklar1997rayleigh}).

\begin{lemma}
	\label{lem:h_i_E}
	Under Assumption \ref{ass:alphas}, $\mathbb{E}[h_{i}(k)] = \frac{1}{N}$ holds for all $i=1,2,\ldots, N$ and $k\geq 0$.
\end{lemma}
\begin{proof}
	From (\ref{eq:eq_asmp_h}), we have  $h_{i}(k)=\frac{\alpha_{i}(k)}{ \sum_{i=1}^{N}\alpha_{i}(k)}$ and $\sum_{i=1}^{N}h_{i}(k)=1$. Since each $\alpha_{i}(k)$ is sampled from the same distribution as stated in Assumption \ref{ass:alphas}, the underlying distribution of  each $h_{i}(k)$ is the same, i.e., $\mathbb{E}[h_i(k)] = \mathbb{E}[h_{j}(k)]$ holds for all $i,j \in \{1,2,\cdots, N\}$ and $\forall k\geq 0$. Moreover, due to the linearity of expectation, we have $$1=\mathbb{E}\Big[\sum_{i=1}^{N}h_{i}(k)\Big] = \sum_{i=1}^{N}\mathbb{E}[h_{i}(k)]=N\mathbb{E}[h_{i}(k)],$$ which implies  $\mathbb{E}[h_{i}(k)]=\frac{1}{N}$ for all $i=1,2,\ldots, N$. 
\end{proof}

\vspace{2mm}
\subsection{Convergence Analysis of the FedCOTA Algorithm}
\vspace{2mm}

First, we present preparatory results needed to show the convergence of the FedCOTA algorithm. Consider
\begin{align}
	\label{eq:C_k}
	C(k)=C_k =1-\eta(k)\mu.
\end{align}

Note that, under Assumption \ref{ass:step_size_TV}, $C_k \geq 0$ holds $\forall k \in \mathbb{N}_0 $. As $\eta (k)$ is decreasing, $C_k$ is increasing in $k$, hence it suffices to show $C_0 \geq 0$. This immediately follows from $C_0 =1-\eta_c\mu$ and $\eta_c \leq \frac{1}{\mu}$ (see Remark \ref{rem:L}), hence $C_0 \geq 0$.
Note furthermore that, as $C_k$ is increasing, $C_k >0$ holds for all $k \geq 1$.

\vspace{1mm}
\begin{lemma}
	\label{lem:C_limit_1}
	Under Assumption \ref{ass:step_size_TV}, 	$\lim_{k\to\infty} (C_{k})^{k} = 0$.\\[-3mm]
\end{lemma}
\begin{proof}
	By letting $C_{k} = 1-\frac{Q}{\sqrt{k+1}}$, where $Q = \eta_c\mu>0$, we write \\[-3mm] 
	\begin{align}
		\label{eq:exp_inf_1}
		\lim_{k\to\infty} (C_{k})^{k} &= \lim_{k\to\infty} \Big(1-\frac{Q}{\sqrt{k+1}}\Big)^{k}. 
		\\[-2mm]  \nonumber
	\end{align}
	
	Note that $\forall x\in \Re$, we have $1+x\leq e^{x}$. Hence, $\forall k \in  \mathbb{N}_0$ \\[-2mm]
	\begin{align}
		\label{eq:exp_inf_1_J1}
		C_k = 1-\frac{Q}{\sqrt{k+1}}\leq e^{-\frac{Q}{\sqrt{k+1}}}.
	\end{align}
	
	Moreover, as noted above, $C_k \geq 0$, hence 
	\\[-2mm]
	\begin{align}
		\label{eq:exp_inf_1_J2}
		0\leq (C_k)^{k} \leq e^{-\frac{Qk}{\sqrt{k+1}}}.
	\end{align}
	
	As the term on the right hand side of (\ref{eq:exp_inf_1_J2}) goes to zero for $k\to\infty$, we have established that $\lim_{k\to\infty} (C_{k})^{k}= 0$.
\end{proof}

\vspace{1mm}
\begin{lemma}
	\label{lem:last_term_pre}
	Under Assumption \ref{ass:step_size_TV}, 	$\mathlarger{\prod}_{k=0}^{\infty}C_k = 0$.
\end{lemma}	\vspace{1mm}
\begin{proof}
	By using the relation $1+x\leq e^{x}$, $\forall x\in \Re$, we write  
	\begin{align}
		\label{eq:product_C_1}
		\nonumber
		\mathlarger{\prod}_{t=0}^{k}C_t &=	\mathlarger{\prod}_{t=0}^{k}\big(1-\eta(t)\mu \big)
		\\
		&\leq \mathlarger{\prod}_{t=0}^{k} \mathrm{e}^{-\eta(t)\mu}  = \mathrm{e}^{-\mathlarger{\sum}_{t=0}^{k}\eta(t)\mu} 
	\end{align}

	Taking the limit as $k\to\infty$ on both sides gives 
	\\[-3mm]
	\begin{align}
		\label{eq:product_C_2}
		\lim_{k\to\infty}\mathlarger{\prod}_{t=0}^{k}C_t \leq  \mathrm{e}^{-\mu \mathlarger{\sum}_{t=0}^{\infty}\eta (t) } = 0
	\end{align}
	\noindent
	since by Assumption \ref{ass:step_size_TV}, $\mu \mathlarger{\sum}_{t=0}^{\infty}\eta (t)=\infty$.
\end{proof}
\vspace{2mm}

\begin{lemma}
	\label{lem:last_term_aux_2_try}
 	Under Assumption \ref{ass:step_size_TV}, for $k\to\infty$,\\
	$$\mathlarger{\sum}_{t=0}^{k-2}\Big( \mathlarger{\prod}_{l=t+1}^{k-1}  C_l \Big)\eta^{2}(t)+\eta^{2}(k-1)$$\\[-1mm] converges to an arbitrarily small positive value $\epsilon$.
\end{lemma}	
\begin{proof}
	See Appendix.
\end{proof}

We are now ready to present the main result.

\begin{thm}
	\label{thm:thm_main_try}
	Let Assumptions \ref{ass:conset}, \ref{ass:f_i}, \ref{ass:step_size_TV}, and  \ref{ass:alphas} hold. Then,  $\exists\theta^{*} \in \Theta^* $ such that 
	for $k\to\infty$, $\mathbb{E}\big[||\theta(k)-\theta^{*}||^{2}\big]$ 
	is arbitrarily small.
\end{thm}

\begin{proof}
	For any $\theta^{*}\in \Theta^{*}$, by using (\ref{eq:local_agent_update}), (\ref{eq:developFEDOTA}), (\ref{eq:eq_asmp_h}), and the non-expansive\footnote{$||\mathbf{P}_\Theta(x)-\mathbf{P}_\Theta(y) ||\leq ||x-y||$ holds for all $x,y \in \Re^m$ if $\Theta$ is a nonempty closed convex set  (see \cite{bertsekas2003convex}).} 
	property of the projection $\mathbf{P}_\Theta$,  we write
	\\[-1mm]
	\begin{align}
		\label{eq:try_eq_1}
		\nonumber
		||\theta(k+1)-\theta^{*}||^2 &=   \bigg|\bigg|\mathbf{P}_\Theta\Bigg(\sum_{i=1}^{N}h_i(k) \theta_i(k) \Bigg) -\mathbf{P}_\Theta\big(\theta^{*}\big)\bigg|\bigg|^2 \\ \nonumber 
		&\leq \bigg|\bigg|  \sum_{i=1}^{N}h_i(k) \bigg( \theta(k) - \eta(k)\nabla f_i\big(\theta(k)\big)\bigg)-\theta^{*} \bigg|\bigg|^2 \\  \nonumber
		&=\big| \big|  \theta(k)- \theta^{*} \big| \big|^{2}
		\\ \nonumber 
		&\qquad -  2\eta(k)  \sum_{i=1}^{N}h_i(k) \nabla f^T_i\big(\theta(k)\big)\big(\theta(k)- \theta^{*}\big) \\  
		& \qquad \qquad \qquad +  \Big| \Big|  \eta(k) \sum_{i=1}^{N}h_i(k)\nabla f_i\big(\theta(k)\big)\Big| \Big|^{2}. 
	\end{align}
	
	Note that boundedness of the constraint set $\Theta$ and Lipschitz continuity of $\nabla f_i\big(\theta(k)\big)$ imply that  $\exists  D_\Theta,M>0$ such that \\[-1mm]
	\begin{align}
		\label{eq:bounds_norms_1_try}
		\big| \big|  \theta(k)- \theta^{*}\big| \big| &\leq   D_\Theta, \\ 
		\label{eq:bounds_norms_2_try}
		\big|\big| \nabla f_i\big(\theta(k)\big) \big|\big| &\leq M.  \\[-3mm] \nonumber
	\end{align}
	
	Then,  an  upper-bound for the last term on the right hand side of (\ref{eq:try_eq_1}) can be written as 
	
	\begin{align}
		\label{eq:try_eq_2}
		\nonumber
		\Big| \Big|  \eta(k)  \sum_{i=1}^{N}h_i(k)\nabla f_i\big(\theta(k)\big)\Big| \Big|^{2}
		&= \eta^{2}(k) \Big| \Big|    \sum_{i=1}^{N}h_i(k)\nabla f_i\big(\theta(k)\big)\Big| \Big|^{2}\\ \nonumber
		&\leq \eta^{2}(k)\sum_{i=1}^{N}h_i(k) \Big| \Big|    \nabla f_i\big(\theta(k)\big)\Big| \Big|^{2}\\
		&\leq  \eta^{2}(k) M^{2},\\[-2mm] \nonumber
	\end{align} 
	\noindent
	which follows from (\ref{eq:eq_asmp_h}), (\ref{eq:bounds_norms_2_try}), and the convexity of the function $|| \cdot ||^{2}$. Subsequently, taking the expectations of both sides of (\ref{eq:try_eq_1}) and using the linearity of expectation results in  
	
	\begin{align}
	\label{eq:try_eq_12}
	\nonumber
	\mathbb{E}\big[||\theta(k+1)-\theta^{*}||^2\big] &\leq\mathbb{E}\big[\big| \big|  \theta(k)- \theta^{*} \big| \big|^{2}\big]
	\\ \nonumber 
	&-2\eta(k)\mathbb{E}\big[  \sum_{i=1}^{N}h_i(k) \nabla f^T_i\big(\theta(k)\big)\big(\theta(k)- \theta^{*}\big)\big] \\  
	& \qquad \qquad \qquad +  \eta^{2}(k) M^{2}, \\[-3mm] \nonumber
	\end{align}
	\noindent
	where the second term on the right hand side can be written as
	
	\begin{align}
		\label{eq:try_eq_122}
		\nonumber
		 &-2\eta(k)\mathbb{E}\big[ \sum_{i=1}^{N}h_i(k) \nabla f^T_i\big(\theta(k)\big)\big(\theta(k)- \theta^{*}\big)\big] 
		\\ 
		&=  -2\eta(k)  \sum_{i=1}^{N}\mathbb{E}\big[h_i(k)\nabla f^T_i\big(\theta(k)\big)\big(\theta(k)- \theta^{*}\big)\big]. \\[-5mm] \nonumber
	\end{align}

	Note that the statistics of $h_i(k)$ ($i=1,2,\cdots, N$) are independent of $h_i(t)$ for $t<k$, which implies that  $\theta (k)$ and  $h_i(k)$ are statistically independent at time $k$ since  the statistics of $\theta (k)$ are dependent only of $h_i(t)$ for $t<k$ and $i=1,2,\cdots, N$. Hence, by using (\ref{eq:def_global}), Lemma \ref{lem:h_i_E}, and the linearity of the expectation, we can further write (\ref{eq:try_eq_122}) as 
	
	\begin{align}
		\label{eq:try_eq_123}
		\nonumber
		&-2\eta(k)  \sum_{i=1}^{N}\mathbb{E}\big[h_i(k)\nabla f^T_i\big(\theta(k)\big)\big(\theta(k)- \theta^{*}\big)\big] \\ \nonumber
		&=  -2\eta(k)  \sum_{i=1}^{N}\mathbb{E}\big[h_i(k)\big] \mathbb{E}\big[\nabla f^T_i\big(\theta(k)\big)\big(\theta(k)- \theta^{*}\big)\big] \\ \nonumber
		&=  -2\eta(k)  \sum_{i=1}^{N}\frac{1}{N} \mathbb{E}\big[\nabla f^T_i\big(\theta(k)\big)\big(\theta(k)- \theta^{*}\big)\big] \\  \nonumber
		&=  -2\eta(k)   \mathbb{E}\big[\sum_{i=1}^{N}\frac{1}{N}\nabla f^T_i\big(\theta(k)\big)\big(\theta(k)- \theta^{*}\big)\big] \\ 		
		& = -2\eta(k)  \mathbb{E}\big[\nabla \mathcal{L}^T\big(\theta(k)\big)\big(\theta(k)- \theta^{*}\big)\big]. 
	\end{align}
	
	Moreover, by Assumption \ref{ass:f_i} (strong convexity of $f_i(\theta)$ and $\mathcal{L}\big(\theta(k)\big)$), we have \\[-2mm]
	\begin{align}
	\label{eq:try_eq_3}
	\nonumber
	-2\eta(k) \nabla \mathcal{L}^{T}\big(\theta(k)\big)\big(\theta(k)-\theta^{*}\big) \leq -&2\eta(k) \big(\mathcal{L}\big(\theta(k)\big) - \mathcal{L}\big(\theta^{*}\big)\big)
	\\ 
	&\redoB{-}  \eta(k)\mu ||\theta(k)-\theta^{*})||^2. 
	\end{align}

	Note that $\mathcal{L}\big(\theta(k)\big) - \mathcal{L}\big(\theta^{*}\big)\geq 0$ holds for all $k\geq 0$ since $\theta^{*}\in\Theta^{*}$ (optimal point in the constraint set), which together with taking the expectations of both sides of (\ref{eq:try_eq_3}) results in \\[-2mm]
	\begin{align}
	\label{eq:try_eq_4}
	-2\eta(k) \mathbb{E}\big[\nabla \mathcal{L}^{T}\big(\theta(k)\big)\big(\theta(k)-\theta^{*}\big)\big] \leq -  \eta(k)\mu\mathbb{E}\big[||\theta(k)-\theta^{*})||^2\big]. 
	\end{align}

	By using (\ref{eq:try_eq_122})-(\ref{eq:try_eq_4}), (\ref{eq:try_eq_12}) can be rewritten as  
\begin{align}
	\label{eq:try_eq_5}
	\nonumber
	\mathbb{E}\big[||\theta(k+1)-\theta^{*}||^2\big] &\leq \Big(1-\eta(k) \mu\Big)\mathbb{E}\big[\big| \big|  \theta(k)- \theta^{*} \big| \big|^{2}\big] \\ \nonumber
	 & \qquad\qquad\qquad+  \eta^{2}(k) M^{2}  \\[2mm]  
	&= C_k\mathbb{E}\big[||\theta(k)-\theta^{*}||^2\big] +  \eta^{2}(k) M^{2}.
\end{align}
The recursive relation (\ref{eq:try_eq_5}) can be written as
\begin{align} 
	\label{eq:theta_k1_6_cor2_try}
	\nonumber
	\mathbb{E}\big[\big| \big| \theta(k)-\theta^{*} \big| \big|^{2}\big] 
	&\leq \bigg( \mathlarger{\prod}_{t=0}^{k-1} C_t \bigg) \mathbb{E}\big[\big| \big|  \theta(0)- \theta^{*} \big| \big|^{2}\big]
	\\
	&+ M^2\Bigg( \mathlarger{\sum}_{t=0}^{k-2}\Big( \mathlarger{\prod}_{l=t+1}^{k-1}  C_l \Big)\eta^{2}(t)+\eta^{2}(k-1)\Bigg).\\[-8mm] \nonumber
\end{align}
Note that by Lemma \ref{lem:last_term_pre},  $\mathlarger{\prod}_{t=0}^{\infty} C_t=0$ holds. Moreover,  by Lemma \ref{lem:last_term_aux_2_try}, Assumption \ref{ass:step_size_TV}, taking the limits of both sides of (\ref{eq:theta_k1_6_cor2_try}) completes the proof.\end{proof}

\vspace{2mm}
\section{Numerical Example}
\label{sec:sim}
\vspace{2mm}

We now apply the FedCOTA algorithm  to a federated logistic regression problem, where a system of agents, each with access to only its own local dataset, tries to accomplish  a global binary classification task.  Let the dataset available to the $i$-th agent be  $D_i = \{d^{n}_i \}_{n=1}^{ |D_i|}$, where $d^{n}_i = (u^{n}_i , z^{n}_i) \in \Re^{m}\times \{0, 1\}$,  and $u^{n}_i$ and $z^{n}_i$ respectively represent input data and labels available to $i$-th agent. Notice that all the agents have 2 different classes of data, labeled by 0 or 1, and their objective is to find a separating hyperplane in $\Re^{m}$ by using the existing data so that the agent is able to separate some unseen data from different classes. In order to accomplish this in a federated manner, each agent utilizes  cross-entropy as its local cost function, which can be expressed as

\begin{align}\label{eq:crossEnt}
	\nonumber
	f_i(\theta,d_i) = \lambda||\theta||_2^{2}-\frac{1}{|D_i|}\Bigg(&\sum_{n=1}^{|D_i|}  z^{n}_i \log\big(S(\theta^{T}u^{n}_i)\Big)
	\\[-1mm]   
	&+  (1-z^{n}_i) \log\big(1-S(\theta^{T}u^{n}_i)\Big)\Bigg) \\[-7mm] \nonumber
\end{align}
\noindent
where $\hat{z}^{n}_i = S(\theta^{T}u^{n}_i)$ is the local estimate of the label $z^{n}_i$ by the $i$-th agent, $S(x)=\frac{1}{1+e^{-x}}$ is the sigmoid function, $||\theta||_2$ is the $L_2$-norm of the parameter vector $\theta$, and $\lambda$ is the regularization hyper-parameter, where $\lambda=0$ means no regularization while  $\lambda=1$ represents maximum level of regularization. Note that $\lambda=0.0001$ has been chosen in simulations, and the overall cost function given in (\ref{eq:crossEnt}) is strongly convex. Moreover, since the parameter vector  $\theta$ is always in a closed and bounded (constraint) set $\Theta$, the second order derivative of (\ref{eq:crossEnt}) is bounded and therefore has a Lipschitz continuous gradient. Additionally, for each agent, the step size is identically chosen as $\eta (k)= \frac{1}{\sqrt{k+1}}$, and the overall cost function is then\\[-2mm]
\begin{align}\label{eq:crossEntGlobal}
	\mathcal{L}(\theta)=\frac{1}{N} \sum_{i=1}^{N} f_i(\theta,d_i)
\end{align}
\noindent
where $N$ is the number of agents in the system. We consider a system of $10$ agents, each having a total number of $|D_i|=100$ training samples. The parameter vector has dimension $m=3$, i.e.,  $\theta\in \Re^{3}$, which also includes a bias term. The objective is to find $\theta^{*} = \argmin_{\theta\in \Theta }  \mathcal{L}(\theta)$, where $\Theta = \{\theta\in \Re^{3}| \hspace{1mm} ||\theta(k) ||\leq 15, \forall k\geq 0\}$ is the constraint set.

\begin{figure}[tb]
	\advance\leftskip -2mm
	\includegraphics[scale=0.5855]{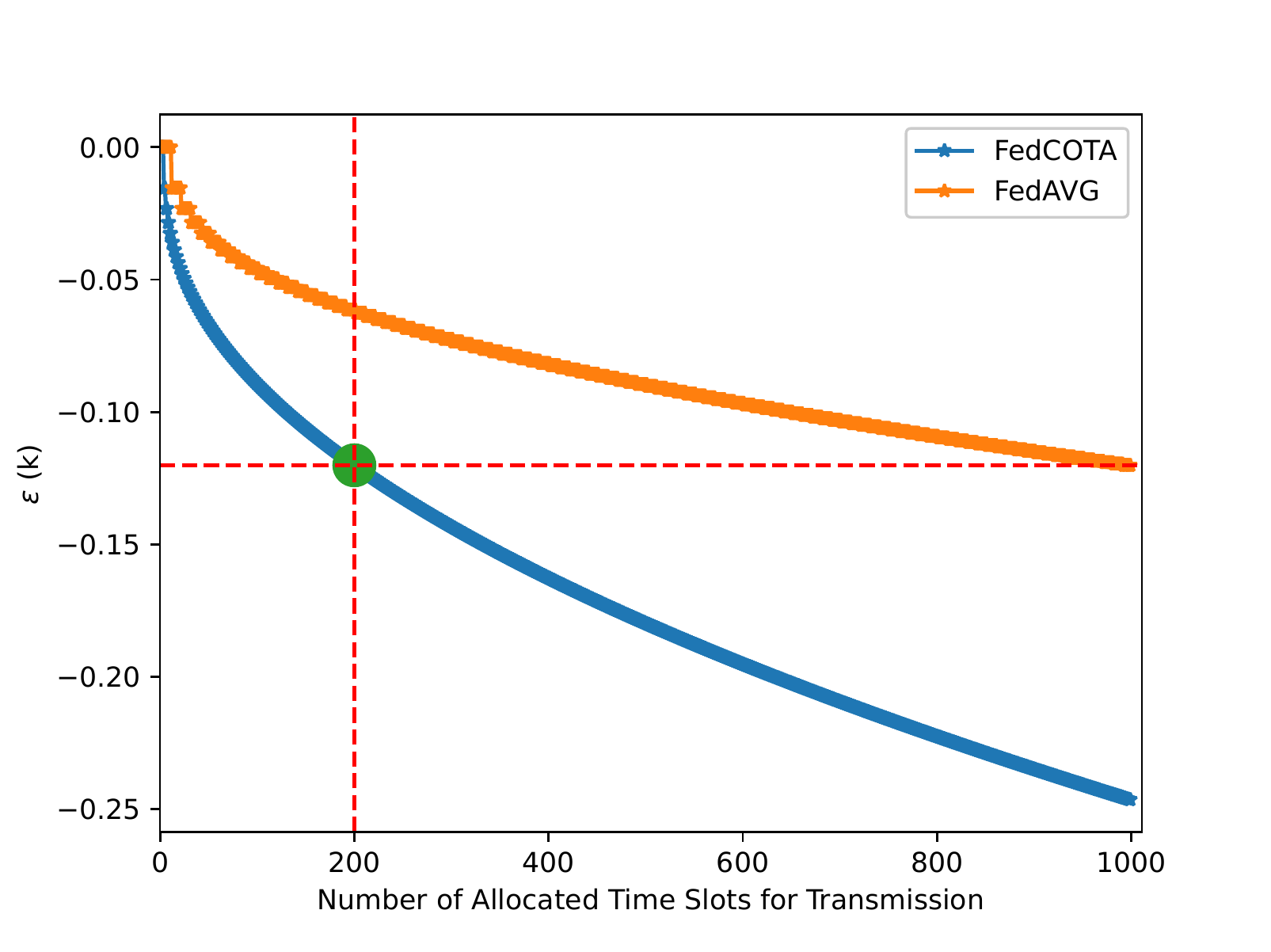}
	\caption{Comparison of a global logistic regression and federated logistic regression models. While the former is trained over the entire dataset, the latter is trained via FedCOTA and FedAVG.}
	\label{fig:param}
\end{figure}

In order to obtain a desired parameter vector $\theta_{d}$, we have trained a global logistic regression model, which has access to the entire dataset. We would expect that the parameter vector $\theta(k)$ tends to converge $\theta_{d}$. This can be assessed by computing $\epsilon(k)=\log_{10}\Big(\frac{\theta(k)-\theta_{d}}{\theta(0)-\theta_{d}}\Big)$ for each communication round $k\geq 0$. As it can be seen in  Fig.  \ref{fig:param}, the parameter vectors of agents utilizing the FedCOTA algorithm tend to converge to $\theta_{d}$.

Note that Fig.  \ref{fig:param} also includes a comparison of the FedCOTA algorithm and the standard FedAVG algorithm \cite{mcmahan2017communication}, for which the TDMA scheme is used for communication between agents and the central unit. In this case, individual time slots are allocated for each agent to transmit its parameter vector at each communication round. After receiving  parameter vectors, the server computes the average of them and sends it back to the agents. Since we consider a system with $N=10$ agents, it takes $10$ time slots per communication round for each agent to transmit its updated parameter vector while only $2$ are needed for the FedaOTA algorithm (see (\ref{eq:ser_rec_1}) and (\ref{eq:ser_rec_2})), which makes it 5 times faster than the FedAVG algorithm (see Fig.  \ref{fig:param}).

\vspace{2mm}
\section{Conclusion}
\label{sec:Conclusion}
\vspace{2mm}

In this paper we have investigated the federated learning problem via Over-the-Air Computation, which provides significant improvements in terms of communication efficiency and privacy. We have investigated the convergence properties of the proposed gradient-based  algorithm by considering time-varying step sizes. Subsequently, we have  presented some numerical examples to illustrate our theoretical results.    

Our current research is on the use of stochastic gradient descent, which is more efficient when we have large number of training samples. In addition, future work will also include cases where the communication between agents is fully distributed and data distribution among users are not independent and identically distributed (non-iid).

\vspace{2mm}
\section{Appendix: Proof of Lemma \ref{lem:last_term_aux_2_try}}
\vspace{2mm}

Due to Assumption \ref{ass:step_size_TV}, $C_k$ is increasing and nonnegative.	Then, we have \\[-2mm]
	\begin{align}
		\label{lem:last_term_aux_2_try_2}
		\mathlarger{\sum}_{t=0}^{k-2}\Big( \mathlarger{\prod}_{l=t+1}^{k-1}  C_l \Big)\eta^{2}(t)+\eta^{2}(k-1)
		\leq \mathlarger{\sum}_{t=0}^{k-1}\big(C_{k-1}\big)^{k-t-1}\eta^{2}(t).
	\end{align}	

	Moreover, again by Assumption \ref{ass:step_size_TV}, $\eta (k)$ is decreasing and $\lim_{k\to\infty}\eta (k)=0$. Thus, for an arbitrarily given $\epsilon > 0$, there exists a time step $k_0>0$ such that $\eta(k)\leq \epsilon$, $\forall k \geq k_0$. Since $\eta (k)$ is decreasing, we have
	
	\begin{align}
		\label{lem:aux_2_try}
		\eta(k)<\eta(k-1)= \frac{1-C_{k-1}}{\mu}.
	\end{align}	
	
	Multiplying both sides of (\ref{lem:aux_2_try}) with $\eta(k)$ provides for all $k\geq k_0$
	
	\begin{align}
		\label{eq:lr_extra}
		\eta^{2}(k)< \frac{(1-C_{k-1}) \epsilon}{\mu}.
	\end{align}

	For  $k> k_0+1$, the right hand side of (\ref{lem:last_term_aux_2_try_2}) can be  bounded as
	
	\begin{align}
		\label{eq:Lem_theta_k1}
		\nonumber 
		&\mathlarger{\sum}_{t=0}^{k-1} \big(C_{k-1}\big)^{k-t-1} \eta^{2}(t)\leq \mathlarger{\sum}_{t=0}^{k_0} \big(C_{k-1}\big)^{k-t-1} \eta^{2}(t) 
		\\ \nonumber
		& \hspace{3.5cm}+ 	\mathlarger{\sum}_{t=k_0+1}^{k-1} \big(C_{k-1}\big)^{k-t-1} \eta^{2}(t).\\[-4mm]
	\end{align}
	
	Decreasingness of $\eta(k)$ implies  $\eta^2(0)\geq \eta^2(k) $ for $k\geq 0$. From (\ref{eq:C_k}), $C_{k-1}\leq 1$ for all $k\geq 1$. Hence, the first term on the right hand side of (\ref{eq:Lem_theta_k1}) can be written as
	
	\begin{align}
		\label{eq:Lem_theta_k1_ext_1}
		\nonumber 
		\mathlarger{\sum}_{t=0}^{k_0} \big(C_{k-1}\big)^{k-t-1} \eta^{2}(t) &\leq \eta^2(0) \mathlarger{\sum}_{t=0}^{k_0} \big(C_{k-1}\big)^{k-t-1} \\ \nonumber
		& = \eta^2(0) \mathlarger{\sum}_{t=0}^{k_0} \big(C_{k-1}\big)^{k-k_0+t-1}\\ \nonumber
		& = \eta^2(0)\big(C_{k-1}\big)^{-k_0}\big(C_{k-1}\big)^{k-1} \mathlarger{\sum}_{t=0}^{k_0} \big(C_{k-1}\big)^{t}\\ 
		& \leq \eta^2(0)\big(C_{k-1}\big)^{-k_0}\big(C_{k-1}\big)^{k-1} (k_0+1), 
	\end{align}
	\noindent
	where $\lim_{k\to\infty} (C_{k-1})^{k-1} = 0$ holds by Lemma \ref{lem:C_limit_1}. Thus, we have
	
	\begin{align}
		\label{eq:Lem_theta_k1_ext_2}
		\lim_{k\to\infty}\mathlarger{\sum}_{t=0}^{k_0} \big(C_{k-1}\big)^{k-t-1} \eta^{2}(t) =0.
	\end{align}

	In addition, (\ref{eq:lr_extra}) holds for $k> k_0 + 1$, which allows us to find an upper-bound for the second term on the right hand side of (\ref{eq:Lem_theta_k1}) as
	
	\begin{align}
		\label{eq:Lem_theta_k1_ext_3}
		\nonumber
		\mathlarger{\sum}_{t=k_0+1}^{k-1} \big(C_{k-1}\big)^{k-t-1}\eta^2(k) &\leq   \frac{(1-C_{k-1}) \epsilon}{\mu} \mathlarger{\sum}_{t=k_0+1}^{k-1} \big(C_{k-1}\big)^{k-t-1}\\ \nonumber
		& = \frac{(1-C_{k-1}) \epsilon}{\mu} \Bigg(\frac{1-\big(C_{k-1}\big)^{k-k_0-1}}{1-C_{k-1}}\Bigg)
		\\
		&\leq \frac{ \epsilon}{\mu}. 
	\end{align}	

	By using (\ref{eq:Lem_theta_k1_ext_1})-(\ref{eq:Lem_theta_k1_ext_3}), and taking the limit of both sides of (\ref{eq:Lem_theta_k1}) results in 
	\begin{align}
		\label{eq:lim_c_2}
		\lim_{k\to\infty}\mathlarger{\sum}_{t=0}^{k-1} \big(C_{k-1}\big)^{k-t-1} \eta(t) \leq \frac{\epsilon}{\mu}.
	\end{align}
	
	Since $\epsilon$ is arbitrary, we complete the proof.

\vspace{3mm}
\balance
\bibliography{ref.bib}
\bibliographystyle{IEEETran}

\end{document}